\documentclass{article}

\usepackage[mathcal]{euscript}
\usepackage[scr=boondoxo,scrscaled=1.05]{mathalfa}
\usepackage{mathtools}
\usepackage{amssymb}
\usepackage{subcaption}
\usepackage{float}
\usepackage[ruled, vlined, linesnumbered]{algorithm2e}
\usepackage{algpseudocode}
\usepackage{amsthm}
\usepackage{dsfont}
\usepackage{booktabs}
\usepackage{multirow}
\usepackage{array}
\usepackage{bm}
\usepackage{xcolor}

\usepackage{setspace}
\usepackage[margin=2cm]{geometry}
\usepackage{newtxtext}
\usepackage{newtxmath}

\usepackage[absolute]{textpos}

\makeatletter
\newcommand{\vast}{\bBigg@{4}}
\newcommand{\Vast}{\bBigg@{5}}
\makeatother

\makeatletter
\newcommand*{\defeq}{\mathrel{\rlap{%
			\raisebox{0.3ex}{$\m@th\cdot$}}%
		\raisebox{-0.3ex}{$\m@th\cdot$}}%
	=}

\newcommand{\zzz}[1]{\textcolor{black}{#1}}

\makeatletter
\def\thmheadbrackets#1#2#3{%
	\thmname{#1}\thmnumber{\@ifnotempty{#1}{ }\@upn{#2}}%
	\thmnote{ {\the\thm@notefont[#3]}}}
\makeatother

\newtheoremstyle{brakets}
{}
{}
{\normalfont}
{}
{\bfseries}
{.}
{ }
{\thmheadbrackets{#1}{#2}{#3}}

\newtheoremstyle{defbrakets}
{}
{}
{\normalfont}
{}
{\bfseries}
{.}
{ }
{\thmheadbrackets{#1}{#2}{#3}}

\newtheoremstyle{defproblem}
{}
{}
{\normalfont}
{}
{\bfseries}
{.}
{ }
{\thmheadbrackets{#1}{#2}{#3}}

\theoremstyle{brakets}
\newtheorem{thm}{Theorem}

\theoremstyle{defbrakets}
\newtheorem{cor}{Corollary}
\newtheorem{lem}[thm]{Lemma}
\newtheorem{defn}[thm]{Definition}
\newtheorem{prop}{Proposition}
\theoremstyle{defproblem}
\newtheorem{plm}[thm]{Problem}
\newtheorem{rem}[thm]{Remark}

\makeatletter
\newcommand{\algrule}[1][.2pt]{\par\vskip.5\baselineskip\hrule height #1\par\vskip.5\baselineskip}
\makeatother

\let\oldnl\nl
\newcommand{\nonl}{\renewcommand{\nl}{\let\nl\oldnl}}

\DeclarePairedDelimiter{\norm}{\lVert}{\rVert}

\newcommand{\xx}[1]{\textcolor{black}{#1}}
\newcommand{\yy}[1]{\textcolor{black}{#1}}
\newcommand{\qq}[1]{\textcolor{black}{#1}}

\begin{document}

	\title{Topology Recoverability Prediction for Ad-hoc Robot Networks: A Data-Driven Fault-Tolerant Approach}
\author{Matin~Macktoobian, Zhan, Shu, and Qing Zhao\footnote{matin.macktoobian@ualberta.ca}}%
\date{Electrical and Computer Engineering Department\\ University of Alberta\\Edmonton, AB, Canada}
\maketitle

\begin{textblock}{14}(2,1)
	\noindent\textbf{\color{red}Published in 
		``IEEE Transactions on Signal and Information Processing over Networks'' \\DOI: 10.1109/TSIPN.2023.3328275}
\end{textblock}

\begin{abstract}
	Faults occurring in ad-hoc robot networks may fatally perturb their topologies leading to disconnection of subsets of those networks. Optimal topology synthesis is generally resource-intensive and time-consuming to be done in real time for large ad-hoc robot networks. One should only perform topology re-computations if the probability of topology recoverability after the occurrence of any fault surpasses that of its irrecoverability. We formulate this problem as a binary classification problem. Then, we develop a two-pathway data-driven model based on Bayesian Gaussian mixture models that predicts the solution to a typical problem by two different pre-fault and post-fault prediction pathways. The results, obtained by the integration of the predictions of those pathways, clearly indicate the success of our model in solving the topology (ir)recoverability prediction problem compared to the best of current strategies found in the literature.
\end{abstract}

\textbf{keywords}: Topology Prediction, Fault-Tolerant Prediction, Ad-hoc Robot Networks\xx{ , Bayesian G\yy{au}ssin Mixture Model, Mobile Robotics}

\maketitle
\doublespacing
\section{Introduction}
Ad-hoc robot networks provide communicational dexterity\footnote{\zzz{Communication dexterity refers to the feasibility of communication in extreme environments cluttered with objects, which may lead to collision faults, and/or in the presence of massive loads of information to be passed through an ad-hoc robot network prone to congestion.}} as well as mobility for many applications such as surveillance \cite{ghedini2018toward}, wireless sensor networks \cite{li2012servicing}, rescue missions \cite{birk2009networking}, aerial transportation \cite{bernard2011autonomous}, and so on. In such multi-robot systems, every pair of robots shall be able to communicate with each other. Thus, a minimum set of links, known as topology, has to be found to suffice the cited connectivity. Topology preservation is assumed to be equivalent with communication preservation in an ad-hoc robot network, as long as its topology is designed based on some constraints that impact communication at a protocol level \cite{chen2019topology}. Given a particular topology, wireless communicational links between nodes of ad-hoc robot networks are heterogeneously subject to noise. The more those disturbances are applied to a link, the less its quality is. Link quality is a major concern in topology construction and recoverability \cite{tardioli2010enforcing}, so that one can impose different constraints \cite{santra2013study} corresponding to shape, length, fan-in/fan-out, etc., on a link selection process to achieve various optimal topology criteria with respect to those constraints.

In the case of some favorite topology geometries\footnote{The geometry of a topology refers to the shape of the graph associated with that topology.}, especially star-shaped ones, the analytical computation of an optimal topology\footnote{\zzz{Optimal topology is a topology that fully satisfies a set of congestion-driven constraints for an ad-hoc robot network or a computer network \cite{macktoobian2023learning}.}} cannot be efficiently scaled with respect to the size of a robot network. In other words, the complexity of finding an analytical solution to an optimal topology problem\footnote{We follow the definition of optimal topology as a class of cycle topologies having the largest possible cycle, with respect to some connectivity constraints, as the backbone, and a set of tree-like branches \cite{macktoobian2023learning}.} of a robot network is equivalent to the NP-hard problem of checking whether the graph of that network is Hamiltonian \cite{mavrogiannis2021hamiltonian}. The quoted computational inefficiency may frequently occur in ad-hoc robot networks because they often constitute many robots. The more robots a network possesses, the more robust it will be in terms of undoing potential fault impacts by their robot redundancies, and in terms of the recoverability of the graph structure associated with its topology after the occurrence of a fault\footnote{\zzz{In this paper, faults refer to either collision or congestion scenarios as a result of which a robot \xx{is} disconnected from the topology of its network.}}. Moreover, larger networks are more capable of performing complicated missions. \xx{However}, optimal topology planning for ad-hoc robot networks cannot be analytically conducted. Even in the case of relatively small networks, including up to 10 robots, given a limited set of optimality criteria, the required analytical computations are extremely demanding and resource-intensive because the mobility of robots requires many iterations of topology computations once they move in their environment. In some applications, one may alleviate such computational bottlenecks at the cost of yielding sub-optimal topologies \cite{ghosh2022cognitive}. However, that can be only possible if a network is totally observable in the course of its operations, which is not a realistic assumption in many scenarios that are entangled with uncertainty and safety concerns. 

\yy{Collisions and congestions may be due to inherent incapabilities of a network's robots to deal with critical coordination and communication scenarios, which can jeopardize their topologies. Another source of topological issues associated with an ad-hoc robot network can be security-driven intrusions \cite{zhang2000intrusion,sarika2016security} according to which an intruding entity poses malicious signals to dissociate their topology \cite{ponnusamy2022intrusion}. These issues are well-studied in the case of weighted graphs using XTC model \cite{wattenhofer2004xtc}. Another known effective strategy is the usage of multi-phase topologies, e.g., \cite{shen2004cltc,bilen2022three}, in which a working topology can be re-spawned by multiple robots in the case of a particular type of intrusion.}

Alternatively, one may take the data associated with former topologies of an ad-hoc robot network to synthesize a machine-learning-based semi-optimal topology which resembles optimal ones with very high accuracies \cite{macktoobian2023learning,zhou2020new}. According to this data-driven point of view, one has to select a feature set of an ad-hoc robot network that can be attributed to an optimal topology via a particular set of constraints and requirements. Many instances of those feature-topology correspondences have to be collected in a dataset. Then, supervised classification methods can be applied to such datasets to achieve models for the prediction of optimal topologies associated with unseen feature sets of that network.

The aforesaid data-driven methods are remarkably more efficient than analytical solutions. However, limited computational resources of an ad-hoc robot network have to be partially dedicated to the real-time employment of these data-driven strategies. One of the prominent situations in which topology re-computations may be necessary is where the structure of an ad-hoc robot network significantly changes because of faults. \xx{In an ad-hoc robot network, one may observe two types of faults, in the course of a network's operations, can lead to fatal impacts on the topology of the network, thereby interfering with its desired communicational characteristics. The first type of fault stems from collisions between robots operating in each others' vicinities. }In these scenarios, a subset of robots may face outage due to fatal faults. In this regard, those outages may jeopardize the validity of the topology of such an ad-hoc robot network and its overall connectivity. In a naive-but-inefficient strategy, one may perform the quoted machine-learning-based predictive computations to yield another working topology. However, the removal of some robots from an ad-hoc robot network may fatally interfere with the connectivity of other robots so that no topology may preserve their connectivity. In such situations, running any topology prediction and/or computation procedure wastes computational resources. So, topology recoverability of a network after a fault occurrence is essential to limit intensive topology re-computations solely to the scenarios in which it is possible to recover the topology via finding an alternative one. In this paper, we seek an accurate computationally-efficient data-driven solution to the following problem.
\begin{plm}
	\label{plm}
	Given an algorithm, associated with a set of
	optimality criteria, to compute the optimal topology corresponding
	to an ad-hoc robot network, assume that a fault
	occurs turning a subset of the robots communicationally dysfunctional. Then, determine whether or not the remaining
	robots can construct a new topology to preserve the connectivity
	of their robot network.
\end{plm}
\yy{The solution to this problem is particularly crucial in the control of multi-robot systems for operations in extreme environments such as space operations \cite{leitner2009multi}, search and rescue \cite{queralta2020collaborative}, extinguishing large-scale wildfire \cite{couceiro2019semfire}, etc. In these applications, due to severe environmental conditions, robots may face outage, but	the assessment of the impact of such outages on remaining robots is vital to properly synthesize control signals for resolving those situations. Thus, because of the complexity of such environments and lack of full observability between every pair of robots, any movement of a robot has to be subject to the preservation of communication between the remainder of robots. This paper exhibits a data-driven solution to this problem. The drawbacks of our strategy, as motivation for future research in this field, are briefly discussed in Section \ref{sec:conc}.}

To solve this binary classification problem, we propose a framework in which topologies are Gaussian embeddings of robot coordinates. We obtain those embeddings by using bi-variate Bayesian Gaussian mixture models (B-BGMMs)\footnote{The foundations of B-BGMMs, as far as required to understand this paper's contribution, are reviewed in Section \ref{subsec:BGMM}.}. The resulting probability density functions (PDFs) may be immediately used for Bayesian inference regarding faults and their impacts on network topologies. The predictive assessments of our method are essentially performed in the scope of robot neighborhoods which exhibits the distributed approach of ours regarding communications among robots. For this purpose, the proposed framework constitutes two separate recoverability assessment pathways associated with faults. The first pathway performs pre-fault prediction which takes the nominal network, as a whole, into account before the occurrence of any fault to generate B-BGMM-based PDFs. Then, given some particular faults occurred in some robots, we develop a Bayesian inference engine to predict the connectivity probability of the post-fault network in the absence of those robots. The second pathway assesses the same target in a post-fault perspective. Namely, PDFs are generated only for the robots that were immediately connected to faulty robots. Then, a decision-maker algorithm combines the assessments of these two pathways to classify the problem to be either recoverable or irrecoverable. If the model votes for the topology recoverability in a post-fault scenario, one may run an efficient data-driven topology synthesizer, for example, OpTopNET \cite{macktoobian2023learning}, to achieve a new topology. Otherwise, further topology computations are avoided until robot coordinates are significantly changed. In that case, the dataset content may no longer be associated with the new coordinates of robots, and one may reach a new valid topology via direct computations. 

We also highlight the point that our learning approach has to be online because of unpredictable dynamics of robots in an ad-hoc robot network in the course of their missions, e.g., in the case of teams of firefighting robots or those used in search applications. Additionally, in our formulation, any factor that jeopardizes the topology of an ad-hoc robot network is considered as a fault, such as the separation of a robot from its peers at distances farther than the communication threshold range of their network. 

The reader may note that the existence of an optimal topology with respect to a post-fault network is equivalent to the recoverability of its corresponding pre-fault topology. \xx{Thus}, the feasibility analysis of optimal topology seeking is an automatic result of our method. In this paper, we assume that faults are captured by some background mechanism \cite{wang2019hybrid}. Our proposed solution is in fact a post-processing procedure applicable to an ad-hoc robot network, should the set of its communicationally-irresponsive robots \xx{be} known. 

\zzz{The proposed strategy is centralized because in fault-sensitive scenarios, especially those including 10 to 20 robots similarly to what our approach is applicable to, the safety of centralized strategies established by data usage may be better experimentally verified via simulations compared to decentralized ones. That is because in the latter case, communication interfaces of decentralized predictors have to be separately investigated for safety purposes, which is not a trivial task. \xx{Another complication stemming from decentralization of this approach is the synchronization of predictions corresponding to all robots. In a decentralized setting, synchronization is necessary to assure that all robots predict the local topologies of their neighborhoods simultaneously based on the most updated positions of their peers. On the other hand, packets of data should be as small as possible in ad-hoc robot networks to minimize both the communication delay and processing labor of robot processors. Since any synchronization mechanism essentially increases the size of packets, one has to carefully design efficient schemes to minimize the quoted burden. }\xx{It means} a central entity is required to constantly collect data of robots for any topology (ir)recoverability prediction. This entity may be a communications protocol \cite{trenkwalder2020swarmcom}. The assimilation of data is equivalent to the creation and maintenance of a dataset that is done by the centralized entity described above. Additionally, our method assumes that robot coordinates are available as inputs. So, the method is independent of any data collection strategy as long as coordinates can be captured. In particular, depending on the size of robots and their computational strength, one may use a variety of coordinate collection schemes such as dead reckoning \cite{brossard2020ai} (for small robots) or SLAM \cite{kegeleirs2021swarm} (for larger robots equipped with more powerful computational resources).}
\subsection{Related Work}
A sizeable literature has emerged in the field of fault recoverability prediction and prognosis in recent years. For example, hidden Poisson Markov models \cite{prasanth2021certain} were used for underwater wireless sensor networks. These networks are static without any autonomous mobility capabilities. \xx{Thus}, the resulting single-shot prediction problem may not pose severe computational complexities contrary to the problem of ad-hoc robot networks. An interesting application of dynamic Bayesian networks is the fault prognosis of planetary rovers \cite{codetta2014dynamic}. Despite the expressivity of this approach in terms of finding fatal modes, its high-level formalism is mostly useful for systems engineering purposes to validate and verify safety in such systems. Localization of fault prediction was also proposed to improve the resilience of underwater robot networks against fatal malfunction \cite{das2017fault}. The family of computational intelligence methods, e.g., particle swarm systems, has been also taken into account to predict irregularities regarding various features of ad-hoc robot networks \cite{harold2019psoblap}. These methods rely on the assumption that dynamics of their desired features are known to a certain extent. Despite the prevalence of the cited strategies, the direct usage of data associated with ad-hoc robot networks for topology (ir)recoverability prediction in the presence of faults have been rarely taken into account.

Machine-learning-based techniques, e.g., long short-term memories \cite{wu2020approach} and dynamic wavelet neural networks \cite{jin2018fault}, are also well-established tools to track topology (ir)recoverability of faults. However, they have been exclusively applied to health monitoring data or similar classes of data exhibiting extreme cross-correlations between their features. Fault recoverability prediction for wind turbines has been studied using deep networks \cite{xu2020fault}. \xx{However}, the efficiency of these models applied to multi-agent cases, such as ad-hoc robot networks, is unclear. In particular, these models usually deal with limited sets of features, while the feature number for ad-hoc robot networks is often large, e.g., proportional to the cardinality of their populations.

Attention-based approaches, e.g., \cite{liu2019attention}, and adversarial methods, e.g., \cite{yang2021fault}, were also used for fault recoverability prediction of industrial robots. However, they fundamentally assume that faults are not fatal which is not generally the case in ad-hoc robot networks. \xx{The embodiment of the attention mechanism to echo state networks led to a powerful predictor which is relatively scalable with respect to the size of a network \cite{moore2001clustering}. Multiple linear regression was a promising strategy to employ attention for fault detection processing \cite{roberts1998bayesian} and prognosis \cite{mosallam2016data}.} BGMMs have been already used for sensor networks \cite{safarinejadian2010distributed}, self-organizing maps \cite{yin2001bayesian}, signal processing \cite{plataniotis2017gaussian}, anomaly detection in hyper-spectral imagery \cite{wang2021review}, robotic policy imitation \cite{pignat2019bayesian}, fault diagnosis of rotary machines \cite{li2021convolutional}, etc. The formulation of \cite{yin2001bayesian} particularly resembles ours in view of the confinement of feature sets to minimal ones.

\subsection{Contributions}
The contributions of this work, characterizing our solution to Problem \ref{plm}, are highlighted as follows.
\begin{enumerate}
	\item Due to the correspondence between robot coordinates in an ad-hoc robot network and its optimal topology \cite{macktoobian2023learning}, our model only takes those coordinates as its input. This minimizes the set of features generated by B-BGMMs yielding higher accuracies and faster convergences. 
	\item Our proposed solution is independent of any fault-detection strategy\xx{, according to which} one may deploy any fault detector subsystem, should it determine those robots in which faults are occurred. 
	\item Our decision-making procedure includes a minimum set of hyperparameters, all of which are both topologically-intuitive and semantically-relative to hyperparameters of topologies. Thus, one can even efficiently set them without any particular tuning process to achieve high classification accuracies.
	\item Because of the probabilistic nature of our algorithm, it can inherently deal with partial observabilities and stochastic nature of robot coordinates. 
	\item Collisions and congestions, as the major sources of faults in ad-hoc robot networks \cite{ramanathan2002brief}, are fully taken into account in the Bayesian inference engine of our scheme.
	\item The differential kernel of the decision-making process of our model, supported by our double-pathway classification mechanism, yields up to 90\% classification accuracy.
\end{enumerate}
\subsection{Organization}
The remainder of this article is structured as below. Section \ref{sec:prel} covers preliminary concepts relevant to the proposed solution. To be specific, Section \ref{subsec:corres} reviews the known correspondence between robot coordinates and the optimal topology of their network. Given the results of this section, we later embed Gaussian transformations of coordinates into B-BGMMs models. A brief formalism of B-BGMMs is presented in Section \ref{subsec:BGMM}. Section \ref{sec:meth} rigorously describes our methodology. For this purpose, we first state the derivation of probabilities based on PDFs of our B-BGMMs. Then, Section\xx{s} \ref{subsec:pre} and \ref{subsec:post} explain the machinery of the pre-fault and post-fault prediction pathways of our model, respectively, while Section \ref{subsec:alg} outlines how results of the cited pathways are integrated into a heuristic algorithm to yield classification results. The results exhibiting the efficiency of our method are represented in Section \ref{sec:res}. We draw our conclusions in Section \ref{sec:conc}.
\section{Preliminaries}
\label{sec:prel}
\subsection{Topology-Coordinate Correspondence}
\label{subsec:corres}
A topology for an ad-hoc robot network comprises a minimum set of links between robots of that network such that every robot can (in)directly communicate with other robots. It is known that one may establish a many-to-one correspondence between robot coordinates in a network, with respect to a set of ad-hoc communications criteria, to a particular topology that optimizes communications among those robots \cite{macktoobian2023learning}. Mathematically speaking, a topology $\mathcal{T}$ can be computed given a robot coordinate set, say, $\mathcal{X} := \{\mathcal{X}_{i} \mid i\in\mathcal{I} \}$, where $\mathcal{I}$ is an index set, via a mapping $\mathcal{T} = f(\mathcal{X})$.

Coordinate set $\mathcal{X}$ is often known when an ad-hoc robot network is deployed. However, it is not necessarily the case when the network may be subject to faults and partial observability. Thus, we consider coordinate set $\mathcal{X}$ as a set of random variables to efficiently take such dynamics into account.
Here, we present some general definitions of the essential notions that are required in the course of representing our approach.
\begin{defn}[Connectivity Threshold]
	Suppose two robots whose coordinates are specified by random variables $\mathcal{X}_{i}$ and $\mathcal{X}_{j}$. Then, given a connectivity threshold $\delta$, they can be directly linked to each other in a topology if the following condition holds\footnote{Operator $\norm{\cdot}$ represents the Euclidean norm.}: $\norm{\mathcal{X}_{i} - \mathcal{X}_{j}} \le \delta$.
\end{defn}
\begin{rem}
	The notion of connectivity threshold \cite{khateri2019comparison} provides a simple radial model for the maintenance of connectivity between any pair of sufficiently-close robots.
\end{rem}
\begin{defn}[Topological Neighbor Set]
	Suppose a robot network $\mathcal{R}$ whose topology links are characterized by the binary relation $\mathcal{V}(\mathcal{X}_{i}, \mathcal{X}_{j})$ when robots coordinated at $\mathcal{X}_{i}$ and $\mathcal{X}_{j}$ are linked together. Given a particular robot $\mathcal{R}_{i} \in \mathcal{R}$, the immediate neighbor set with respect to $\mathcal{R}_{i}$ is defined as $\Xi_{i} := \{\mathcal{X}_{j} \mid \mathcal{V}(\mathcal{X}_{i}, \mathcal{X}_{j}) \}$.
\end{defn}
\begin{defn}[Connectivity Relation]
	Given a set of robots $\mathcal{R}$ representing a robot network, unary relation $\mathcal{C}(\mathcal{R})$ holds if the robots of $\mathcal{R}$ are communicationally connected to each other by a topology. 
\end{defn}
\begin{defn}[Orphan Robot Set]
	Let $\mathcal{R}$ be a topologically connected ad-hoc robot network, i.e., $\mathcal{C}(\mathcal{R})$ holds. Assume that a fault leads to the outage of $\mathcal{R}_{i} \in \mathcal{R}$. Given the topological neighbor set $\Xi_i$ associated with $\mathcal{R}_{i}$, the orphan robot set $\mathcal{O}_{i} \subseteq \Xi_{i}$ is a subset of the topological neighbors of $\mathcal{R}_{i}$ that lose their connectivities to the rest of the network due to the outage of $\mathcal{R}_{i}$.
\end{defn}
\subsection{Bivariate Bayesian Gaussian Mixture Formalism}
\label{subsec:BGMM}
Bayesian Gaussian mixture models (BGMMs) \cite{roberts1998bayesian} generally use an expectation-maximization criterion to fit data to superposed Gaussian distributions for the purpose of Bayesian inference. The probability distribution function (PDF) associated with a BGMM may be generally multivariate. The number of variables in such a PDF technically depends on particular representations of topological states of a network. In this work, we use bivariate BGMMs (B-BGMMs) for the solution pathways of our problem. The PDF of a B-BGMM may be framed as $\phi_{ij} := \sum_{s \in S}\alpha_{ij}^{s} \mathscr{N}(\mathcal{X}_{ij};\mu_{ij},\Sigma_{ij})$, where $s$ belongs to an index set $S$, $\alpha_{ij}^{s}$ is the weight corresponding to the $s$\xx{-}th component of $\phi_{ij}$. The elements of each bivariate Gaussian distribution $\mathscr{N}(\cdot;\cdot,\cdot)$ are random variables $\mathcal{X}_{ij} :=\big[\mathcal{X}_{i}~~\mathcal{X}_{j}\big]^\intercal$, a mean vector $\mu_{ij} := \big[\mu_{i} ~~ \mu_{j}\big]^\intercal$, and a covariance matrix $\Sigma_{ij} := \mathbb{E}\big[(\mathcal{X}_{i} - \mu_{i})(\mathcal{X}_{j} - \mu_{j})^\intercal\big]$ associated with its robots. These weights are known a priori, but they are constantly updated in the training phase of their underlying BGMM and/or any modification applied to their input variables. \zzz{Tuned weights for an ad-hoc robot network, associated with a particular prediction scenario, may be used to initialize other scenarios. However, if such information is unknown, for example because the network is recently deployed, then random initialization of weights \cite{wilson2020bayesian} is a common strategy to follow, which is also employed by our method.} One may use a Dirichlet distribution (resp., Dirichlet process) to generate a finite (resp., an infinite) mixture model. Infinite models are usually approximated by Stick-breaking representations \cite{griffin2011stick,dunson2008kernel} that are truncated distributions with maximum number of components. These weights, as well as a concentration factor, directly impact the shapes of contours associated with embeddings of their BGMMs. In particular, the larger the concentration factor is, the more active Gaussian harmonics exist in their corresponding PDF.
\section{Methodology}
\label{sec:meth} 
B-BGMMs are often known as clustering methods. However, thanks to their powerful embedding capabilities, they may even offer expressive representations of data when one intends to keep track of very complex data including many interrelated components and variables. In this research, we particularly employ B-BGMMs as transformers of positional-topological data of ad-hoc robot networks both before and after fault occurrences in some of their robots. Assume that a fault occurs in a robot. We are interested in investigating the impact of that robot's absence on the topological connectivity of its network. For this purpose, as sketched in Figure \ref{fig:approach}, we establish two different inference pathways and analyze the connectivity of the overall network based on the comparison of their final connectivity probabilities. This idea comes from the fact that faults perturb covariance matrices of the PDFs associated with nominal topologies. Accordingly, in the course of a \textit{pre-fault prediction}, we use the total data of a nominal network corresponding to some states of the network before a fault occurs in one of its robots. Then, using Bayesian inference, we define two types of faults based on collisions and congestions, and we compute the probabilities, known as marginal likelihood, of their occurrences in that particular robot. We also compute prior probability of the network that states the probability of the connectivity
of the nominal network. By computing the likelihood as the last remaining step, we can achieve the posterior probability of connectivity after the occurrence of a fault.
\begin{figure}
	\centering\includegraphics[scale=0.9]{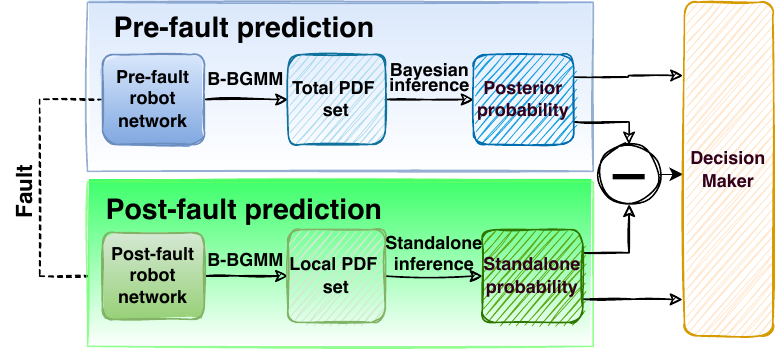}
	\caption{The BGMM-based fault-tolerant topology (ir)recoverability prediction\xx{.} (Algorithm 1, in Section \ref{subsec:alg}, realizes the processing flow depicted in this figure.)}
	\label{fig:approach}
\end{figure}

The resulting probability above has to be finally interpreted as a classification task, say, whether or not the post-fault network may be reconnected again using a viable topology. One may define a single heuristic threshold for this purpose, but such a single parameter may not give rise to very high prediction accuracies. Instead, we seek another estimation, i.e., \textit{post-fault prediction}, of the posterior probability cited above to make the classification more accurate. In particular, the impact of a faulty robot is locally applied to its neighbors. Those neighbors that are solely connected to that robot may lose their connections to the network once a fault occurs in that robot. So, we seek the probabilistic estimation of the connectivity of those orphan robots with the potential peers that are close enough to them. If those orphan robots can be communicationally merged again into the network by a new topology, then the overall network's connectivity is guaranteed. One notes that pre-fault and post-fault predictions are probabilistically equivalent. We use their difference and also their values separately to yield a very efficient classifier with high accuracies.
\subsection{Bayesian Inference through Gaussian Mixtures} 
\label{subsec:Bays}
One may train a B-BGMM using some input data \xx{to} generate a class of PDFs. The derived PDFs of a B-BGMM can be integrated to compute various probabilities related to their corresponding B-BGMM. Particularly, we are interested in the probability of the topological connectivity of two arbitrary robots indexed by $i$ and $j$, i.e., $\mathds{P}\Big[\norm{\mathcal{X}_{i} - \mathcal{X}_{j}} \le \delta\Big]$, that may be written as 
\begin{equation}
	\label{eq:pdf}
	\mathds{P}\Big[\norm{\mathcal{X}_{i} - \mathcal{X}_{j}} \le \delta\Big] = \iint\limits_{\Omega} \phi_{ij} \mathrm{d}\mathcal{X}_{i}\mathrm{d}\mathcal{X}_{j},
\end{equation}
where the disk region $\Omega$ is radially defined in the range of a connectivity threshold around the robots:
\begin{equation}
	\label{eq:region}
	\Omega := \Big\{(\mathcal{X}_{i}, \mathcal{X}_{j})\mid \norm{\mathcal{X}_{i} - \mathcal{X}_{j}} = \delta\Big\}.
\end{equation}
The PDF family $\phi_{ij}$ transforms the data of network topologies into Gaussian embeddings of robot coordinates for the applied analyses in this paper.

In the next sections, given a particular ad-hoc robot network and its corresponding topology, the probability template above plays a central role in the (ir)recoverability estimation of the topological connectivity of that network, should a fault remove one of its robots. Readers interested in more technical details about the BGMM formulation may refer to the related literature, e.g., \cite{li2008simultaneous,svensen2005robust,moore2001clustering,roberts1998bayesian}.
\subsection{Pre-Fault Prediction}
\label{subsec:pre}
As described in the previous section, pre-fault prediction requires the probabilistic assessment of two types of faults stemming from collisions and congestions with respect to a faulty robot. The goal is finding a posterior probability of connectivity through Bayesian inference. To formulate the ingredients of such an analysis, one has to first customize the general formulation of PDFs described in (\ref{eq:pdf}) and (\ref{eq:region}). We note that all robots of a network are subject to this analysis because prior probability computation, as we will later see in detail, requires the iteration over connectivities of robots that can be potentially paired with each other in a topological manner. Indices of the PDFs associated with a pre-fault prediction may arbitrarily vary over the index range of all robots as long as no pair of robots violates the connectivity threshold of their network.

We intend to analyze topology dynamics to track the connectivity of an ad-hoc robot network subject to a fault. So, fault occurrence is formally described in the language of probability theory.
\begin{defn}[Fault Relation]
	Suppose a robot $\mathcal{R}_{i} \in \mathcal{R}$ whose coordinate is denoted by random variable $\mathcal{X}_{i}$. Then, if a fault occurs in the robot, unary relation $\mathcal{F}(\mathcal{X}_{i})$ holds.
\end{defn}    
Mathematically speaking, we seek the probability of the topological connectivity of a robot network, should one of its robots be enforced to leave the topology because of a fatal fault. For this purpose, we employ Bayesian inference to estimate the posterior probability $\mathds{P}[\mathcal{C}(\mathcal{R}) \mid \mathcal{F}(\mathcal{X}_{i})]$, in which random variable $\mathcal{X}_{i}$ represents the faulty robot. Accordingly, we have to build Bayesian constituents, i.e., prior probability, marginal likelihood, and likelihood, for the estimation of the cited posterior probability.
\begin{lem}[Prior Probability]
		\label{lem:PP}
		Suppose $\delta$ denotes the connectivity threshold associated with an ad-hoc robot network $\mathcal{R}$. Given the the topological neighbor set $\Xi_{i}$ associated with each $\mathcal{R}_{i} \in \mathcal{R}$ and $k\defeq|\mathcal{R}|$, the connectivity probability of the network's current topology is
		\begin{equation}
			\label{eq:PP}
				\mathds{P}\Big[\mathcal{C}(\mathcal{R})\Big] =
				\mathbb{P}[\Xi_{k-1}|\Xi_{k-2}, \cdots, \Xi_{1}]\times\cdots\times
				\mathbb{P}[\Xi_{3}|\Xi_{2}, \Xi_{1}]\times\mathbb{P}[\Xi_{2}|\Xi_{1}]\times\mathbb{P}[\Xi_{1}].
		\end{equation}
	\end{lem}
	\begin{proof}
		$\mathcal{R}_{i}$ has to be in the vicinity of its neighboring peers to be a part of the current topology of their network. One notes that the probability of existence of $\Xi_i$ affects the probability of existence
		of $\Xi_j$ for all $j\in\Xi_i$. Therefore, taking the stated dependency into account, we have $\mathds{P}[\mathcal{C}(\mathcal{R})] =\mathbb{P}[\Xi_{k}|\Xi_{k-1}, \cdots, \Xi_{1}]\times
		\mathbb{P}[\Xi_{k-1}|\Xi_{k-2}, \cdots, \Xi_{1}]\times\cdots\times
		\mathbb{P}[\Xi_{3}|\Xi_{2}, \Xi_{1}]\times
		\mathbb{P}[\Xi_{2}|\Xi_{1}]\times\mathbb{P}[\Xi_{1}]$. The existence of the first conditional is already established by all the previous conditionals, i.e., $\mathbb{P}[\Xi_{k}|\Xi_{k-1}, \cdots, \Xi_{1}]=1$. Thus, the lemma's claim holds.
\end{proof}
As noted before, one common type of fault is due to collisions between robots. In particular, in view of a typical robot, locations of its peers may only be partially observable. In other words, the understanding of a robot regarding the coordinates of its neighboring robots may be only expressed in a stochastic manner at some particular moments. \yy{In structured environments, robots often employ sensors to avoid collisions. However, in the case of large-scale ad-hoc robot networks, collision is much more complicated for two reasons. First, robots of these networks are often small with less computational power. So, they rarely utilize state-of-the-art proximity detection sensors. The resulting partial observability may lead to collisions. Even in the presence of such sensors, environments in which these networks work are often quite unstructured and not all complicated dynamics of those environments are known beforehand. For example, in a wildfire extinguishing mission, robots may have to move in a very tight formation during which severe collision avoidance measures may totally stall their movements.} Since occurrences of collisions are more likely when robots are close to each other, we define a threshold based on which collision probabilities can be quantified, as follows.   
\begin{defn}[Collision Threshold]
	Collision threshold $0<\omega<\delta$ determines a lower-bound distance between a pair of robots according to which they may safely move in each other's vicinity without any collision risk.
\end{defn}
\begin{lem}[Marginal Likelihood]
	\label{lem:ML}
	Let $\omega$ be a collision threshold corresponding to an ad-hoc robot network $\mathcal{R}$. Then, the collision probability of robot $\mathcal{R}_{i}\in \mathcal{R}$ is
	\begin{equation}
		\label{eq:ML}
		\mathds{P}\Big[\mathcal{F}(\mathcal{X}_{i})\Big] = \mathds{P}\Bigg[\bigcup_{j \in \Xi_{i}}\Big[\norm{\mathcal{X}_{i} - \mathcal{X}_{j}} < \omega\Big]\Bigg],
	\end{equation}
	where $\Xi_{i}$ is the topological neighbor set of $\mathcal{R}_{i} \in \mathcal{R}$.
\end{lem}%
The second type of faults relates to computational outages of the robots whose communicational loads exceed a certain limit. Such faults are immediate consequences of topology congestions in which a robot is interconnected to many other peers. In such scenarios, similarly to star-shape topologies, a highly-interconnected robot may face failures because of the excessive passage of data associated with its neighboring peers. Topologically speaking, given a relatively-dense neighborhood of robots, the closer a robot to the center of mass of its neighborhood is, the more likely it is to become a hub-like node in that neighborhood, thereby being connected to many of its peers in that neighborhood. We first precisely present the notion of center of mass associated with a neighborhood of robots around a typical robot. We then define a congestion threshold to quantitatively monitor congestions.   
\begin{defn}[Center of Mass]
	Given a robot $\mathcal{R}_{i} \in \mathcal{R}$ and a constant mass factor $0<d<1$, fix the $d$-neighborhood of $\mathcal{R}_{i}$ as $\mathcal{B}_{i}^{d} := \{\mathcal{X}_{j} \vert \mathds{P}[\norm{\mathcal{X}_{i} - \mathcal{X}_{j}} \le \delta]  > d \} \dot{\cup} \{\mathcal{X}_{i}\}$. Then, the center of mass $\overline{\mathcal{X}}_{i}$ of that $d$-neighborhood is formulated as $
	\overline{\mathcal{X}}_{i} := (\sum_{\mathcal{X} \in \mathcal{B}_{i}^{d}}\mathcal{X})/\lvert\mathcal{B}_{i}^{d}\rvert
	$   
\end{defn}
\begin{defn}[Congestion Threshold]
	Congestion threshold $0<\lambda < \delta$ denotes a lower-bound distance between a robot and the center of mass associated with its $d$-neighborhood, for a given $d$.
\end{defn}
\begin{rem}
	\label{rem:dense}
	In dense ad-hoc robot networks, collision is often more critical compared to congestion. Thus, collision threshold should be larger than that of congestion \yy{so that the robots \qq{do not} get excessively close to each other.} On the other hand, in sparse networks in which robots are relatively far from each other, the major bottleneck would be congestion rather than collision. \xx{Accordingly}, congestion threshold should be smaller than that of collision. \yy{Theoretically, there exists efficient routing protocols to manage congestion in networks \cite{chen2007congestion}. However, these protocols often require centralized full observability over the traffic dynamics of a whole network, which is generally not possible in ad-hoc robot networks due to their distributed nature. Moreover, these techniques often require non-trivial computational resources which are not always available in networks encompassing small-scale robots.}
\end{rem}
\yy{\begin{lem}[Likelihood]
		\label{lem:L}
		Let $\delta$, $\omega$, and $\lambda$ be the thresholds associated with the connectivity, collision, and congestion of a robot network, respectively. Then, the likelihood of a fault occurrence in robot $\mathcal{R}_{i} \in \mathcal{R}$ is
		\begin{equation}
			\label{eq:L}
			\mathds{P}\Big[\mathcal{F}(\mathcal{X}_{i})\mid\mathcal{C}(\mathcal{R})\Big] = A + B - AB
		\end{equation}
		where $\Xi_{i}$ is the set of the topological neighbor set of $\mathcal{R}_{i}$ and $A = \mathds{P}\Big[\bigcup_{j \in \Xi_{i}}\Big[\norm{\mathcal{X}_{i} - \mathcal{X}_{j}} < \omega\Big]\Big]$ and $B = \mathds{P}\Big[\norm{\mathcal{X}_{i} - \overline{\mathcal{X}}_{i}} < \lambda\Big]\mathds{P}\Big[\bigcup_{j \in \Xi_{i}}\Big[\norm{\mathcal{X}_{i} - \mathcal{X}_{j}} \le \delta\Big]\Big]$
	\end{lem}
	\begin{proof}
		If $\mathcal{R}$'s topology is connected, then both collisions and congestion faults may (simultaneously) occur in $\mathcal{R}_{i}$. Thus, the fault likelihood of $\mathcal{R}_{i}$ is the summation of the probabilities of both fault types from which their intersection has to be subtracted. Namely, Lemma \ref{lem:ML} provides the marginal likelihood associated with a collision fault. For a congestion fault to occur in $\mathcal{R}_{i}$, it has to be in a close vicinity of its neighbors. Furthermore, the center of mass of the neighborhood has to be close enough to $\mathcal{R}_{i}$. The simultaneous realization of these conditions is equivalent to the assessment of the claimed  probability $\mathds{P}[\norm{\mathcal{X}_{i} - \overline{\mathcal{X}}_{i}} < \lambda]\mathds{P}[\bigcup_{j \in \Xi_{i}}[\norm{\mathcal{X}_{i} - \mathcal{X}_{j}} \le \delta]]$. Moreover, since the two types of faults are independent of each other, the probability of their intersection is the product of their probabilities. Hence, the overall likelihood is as the lemma's claim. 
\end{proof}}
\begin{rem}
	$\mathcal{C}(\mathcal{R})$ and $\mathcal{F}(\mathcal{X}_{i})$ are positively correlated because the more connected an ad-hoc robot network is, the higher the occurrence risk of congestion faults will be. It manifests the essence of the notion of the optimal topology \cite{macktoobian2023learning} as a trade-off to yield the minimum connectivity which provides complete coverage for robots.
\end{rem}
\begin{cor}[Pre-Fault Prediction (Posterior Probability)]
	\label{res:PP}
	Let $\delta$, $\omega$, and $\lambda$ be the thresholds associated with the connectivity, collision, and congestion of an ad-hoc robot network, respectively. Let also $\Xi_{i}$ be the topological neighbor set of $\mathcal{R}_{i} \in \mathcal{R}$. Then, the \textit{pre-fault prediction} (posterior probability) $\mathcal{P}_{\text{pre}}$ of the connectivity of the network's topology subject to a fault occurred in $\mathcal{R}_{i}$ is the immediate Bayesian inference result of Lemma\xx{s} \ref{lem:PP}, \ref{lem:ML}, and \ref{lem:L}, as prior probability, marginal likelihood, and likelihood, respectively.
\end{cor}
\begin{rem}
	One notes that the result above can be straightforwardly generalized to obtain the posterior probability of topology (ir)recoverability if more than one robot experience faults. Namely, let $\mathcal{R}_{\mathcal{F}} \subset \mathcal{R}$ represent a set of faulty robots. Then, the total posterior probability turns into $\mathds{P}[\mathcal{C}(\mathcal{R})\mid\bigcup_{\mathcal{X}_{i} \in \mathcal{R}_{\mathcal{F}} }\mathcal{F}(\mathcal{X}_{i})]$.
\end{rem}
Corollary \ref{res:PP} indeed completes the computation of the pre-fault prediction pathway of Figure \ref{fig:approach}. 
\subsection{Post-Fault Prediction}
\label{subsec:post}
Post-fault prediction topologically seeks the probabilities of shaping new connections between orphan robots and the remainder of their peers. Put differently, the occurrence of a fault is communicationally problematic if the robots connected to a faulty robot can no longer establish any connection to other robots.

In a computational point of view, particularly in the case of large ad-hoc robot networks, connectivity checking for a non-trivial number of orphan robots is resource-intensive. Such a problem may become even more intractable if their topology is star-shape, and a fault occurs in a hub robot, which in turn, gives rise to the emergence of a noticeable number of orphan robots. In such scenarios, if a topology is irrecoverable due to some faulty robots, then spending time and resources to compute a new topology, in which all orphan robots are reconnected to the remainder of their peers, is futile. \xx{Thus}, it is beneficial if one can first estimate the (ir)recoverability probability of a faulty scenario, so that if that probability surpasses a minimum threshold, then topology re-computations shall be performed. 

Post-fault prediction requires the analysis of topology perturbations in neighborhoods of orphan robots that are subject to faults' communicational impacts. Accordingly, one-to-one assessments of those impacts between orphan robots and their neighbors can indeed pave the way for more efficient topological fault analysis. Similarly to the pre-fault prediction scenario, we need to determine the index variations for the general formulations of PDFs, say, (\ref{eq:pdf}) and (\ref{eq:region}). Accordingly, we again take B-BGMMs into account, the first component of which is an orphan robot's topological embedding. The second component of a B-BGMM comprises the coordinate random variable of one of that orphan robot's neighbors subject to some connectivity requirements. The result below provides the post-fault prediction associated with an arbitrary fault scenario.
\begin{thm}[Post-Fault Prediction]
		Suppose $\delta$ denotes the connectivity threshold associated with an ad-hoc robot network $\mathcal{R}$. Assume that a fault leads to the outage of $\mathcal{R}_{i} \in \mathcal{R}$ whose orphan robot set is $\mathcal{O}_{i}$. Let $\mathcal{S}_{i}^{j}$ be the immediate neighborhood set of robot $j\in\mathcal{O}_{i}$.Then, given $k\defeq|\mathcal{O}_{i}|$, the probability of the connectivity of $\mathcal{O}_{i}$ is given by
		\begin{equation}
			\label{eq:post}
				\mathds{P}\Big[\mathcal{C}(\mathcal{O}_{i})\Big] =\mathbb{P}[\mathcal{S}^{k}_{i}|\mathcal{S}_{i}^{k-1}, \cdots, \mathcal{S}_{i}^{1}]\times\mathbb{P}[\mathcal{S}_{i}^{k-1}|\mathcal{S}_{i}^{k-2}, \cdots, \mathcal{S}_{i}^{1}]\times\cdots\times
				\mathbb{P}[\mathcal{S}_{i}^{3}|\mathcal{S}_{i}^{2}, \mathcal{S}_{i}^{1}]\times
				\mathbb{P}[\mathcal{S}_{i}^{2}|\mathcal{S}_{i}^{1}]\times\mathbb{P}[\mathcal{S}_{i}^{1}]
		\end{equation}
	\end{thm}
	\begin{proof}
		The argument resembles that of Lemma \ref{lem:PP}.
\end{proof}
\begin{algorithm}
	\SetNoFillComment
	\caption{Topology (Ir)recoverability Predictor}
	\label{alg:TPP} 
	\SetKwInput{KwData}{Inputs}
	\SetKwInput{KwResult}{Output}
	\KwData{Robot network $\mathcal{R}$\\
		\hspace*{13.4mm}Top\xx{o}logical neighbor set $\mathcal{T}$\\
		\hspace*{13.4mm}Connectivity threshold $\delta$\\
		\hspace*{13.4mm}Collision threshold $\omega$\\
		\hspace*{13.4mm}Congestion threshold $\lambda$\\
		\hspace*{13.4mm}Mass factor $d$\\
		\hspace*{13.4mm}Faulty robot $\mathcal{R}_{f} \in \mathcal{R}$\\
		\hspace*{13.4mm}Bound decision threshold $q_{b}$\\
		\hspace*{13.4mm}Differential decision threshold $q_{d}$}
	\KwResult{\zzz{Predicted class associated with the} topology (ir)recoverability of $\mathcal{R}\!\setminus\!\{\mathcal{R}_{f}\}$}
	\algrule[1pt]
	\tcc{Pre-fault prediction}
	$\Bigl\{\Xi_{i}\Bigr\}_{i=1}^{\lvert\mathcal{R}\rvert} \leftarrow$ Get the topological neighbor set of all elements of $\mathcal{R}$\\
	$\Bigl\{\phi_{i,j}\Bigr\}_{j=1}^{\lvert\Xi_{i}\rvert} \leftarrow$ Compute B-BGMMs for potentially-communicating pairs of robots\\
	Compute prior probability (\ref{eq:PP})\\
	Compute marginal likelihood (\ref{eq:ML})\\
	Compute likelihood (\ref{eq:L})\\
	$\mathcal{P}_{\text{pre}} \leftarrow $Compute pre-fault prediction\\
	\tcc{Post-fault prediction}
	$\Bigl\{\Xi_{f}\Bigr\}_{i=1}^{\lvert\mathcal{R}_{f}\rvert} \leftarrow$ Get the topological neighbor set $\mathcal{R}_{f}$\\
	$\Bigl\{\phi_{f,j}\Bigr\}_{j=1}^{\lvert\Xi_{f}\rvert} \leftarrow$ Compute B-BGMMs for pairs $\Bigl\{(\mathcal{R}_{f}, \mathcal{R}'_{f}) \mid (\forall\mathcal{R}'_{f} \in \Xi_{f})\Bigr\}$\\
	$\mathcal{P}_{\text{post}} \leftarrow $Compute post-fault prediction\\
	\tcc{Decision making}
	\nonl\If{$\mathcal{P}_{\text{pre}} \ge q_{b}\quad \&\quad \mathcal{P}_{\text{post}} \ge q_{b}\quad\&\quad \lvert\mathcal{P}_{\text{pre}} - \mathcal{P}_{\text{post}}\rvert\le q_{d}$}{The topology associated with $\mathcal{T}$ is fault tolerant, thereby being preserved. Now, re-computation of the topology probabilistically yields non-empty results.}
	\nonl\Else{The topology associated with $\mathcal{T}$ is not fault tolerant, so it may not be recovered.}
\end{algorithm}
\begin{rem}
	Since the connectivity bottleneck of an ad-hoc robot network, once a fault occurs, is the connectivity of its orphan robots associated with that fault, then one concludes that post-fault prediction has to be probabilistically the same as the pre-fault prediction, i.e., $
	\mathds{P}[\mathcal{C}(\mathcal{R})\mid\mathcal{F}(\mathcal{X}_{i})] \approx \mathds{P}[\mathcal{C}(\mathcal{O}_{i})]$. Because of the probabilistic nature of the presented analysis, the quoted equivalence is generally unlikely to turn into an equality. However, it opens a new venue, i.e., differential probability, for decision making about how to transform the values of pre-fault and post-fault predictions into one of the desired classes of Problem \ref{plm}. 
\end{rem} 
\subsection{Integration and Classification Algorithm}
\label{subsec:alg}
We need to transform the numerical values of pre-fault and post-fault predictions into categorical class labels of topology (ir)recoverability prediction, which are either recoverable or irrecoverable. For this purpose, we propose the following heuristic rule.
\begin{defn}[Decision-Making Mechanism]
	Denote by $\mathcal{P}_{\text{pre}}$ and $\mathcal{P}_{\text{post}}$ the pre-fault and post-fault predictions associated with the fault-tolerant topology (ir)recoverability prediction of an ad-hoc robot network subject to a fault occurred in one of its robots. Given \textit{bound decision threshold} $0 < q_{b}< 1$ and \textit{differential decision threshold} $0 < q_{d}< 1$, if $\mathcal{P}_{\text{pre}} \ge q_{b}$, $\mathcal{P}_{\text{post}} \ge q_{b}$, and $\lvert\mathcal{P}_{\text{pre}} - \mathcal{P}_{\text{post}}\rvert\le q_{d}$ are (resp., are not) simultaneously satisfied, then the topology of the network is recoverable (resp., irrecoverable).
\end{defn}  

Overall, the total process of the fault-tolerant topology (ir)recoverability prediction is encoded into Algorithm \ref{alg:TPP} that solves Problem \ref{plm}. This algorithm basically computes what Figure \ref{fig:approach} illustrates. \zzz{The algorithm is executed by a centralized software entity that governs the communications protocol of an ad-hoc robot network once a fault occurs. \xx{It means} the frequency of its execution is the same as the frequency of fault occurrences.} In this algorithm, following the formalism developed throughout Section \ref{sec:meth}, one computes pre-fault prediction $\mathcal{P}_{\text{pre}}$ (resp., post-fault prediction $\mathcal{P}_{\text{post}}$) in lines 1-6 (resp., 7-9). The topology (ir)recoverability classification is then conducted according to the decision-making rule explained in this section based on bound and differential decision thresholds.

\xx{The proposed method generally assumes that the coordinate information associated with each robot is certain. However, adding uncertainty to those coordinate information may still lead to correct classifications if the average of noise is significantly smaller than the collision threshold $\omega$. In this case, embeddings generated by B-BGMMs may effectively resemble those generated in the absence of the noise in view of safety and collision avoidance. Nevertheless, the general reliability subject to severe noise and/or partial observability may be a topic of future research. Additionally, the input of the algorithm is indeed a static snapshot associated with position of robots in a dynamic ad-hoc robot network in addition to the positional information of a faulty robot. This can similarly be the case for an immobile sensor network. The difference is that due to the mobility of an ad-hoc robot network, one may need multiple predictions over time, while such successive predictions may not be required for \yy{immobile sensor networks}.}
\begin{table}
	\centering
	\caption{Cross-validation-based tuned values of the topological hyperparameters}
	\begin{tabular}{lll}
		\toprule  
		{\bfseries Parameters}& {\bfseries Variation sets} &{\bfseries Optimal val.}\\\cmidrule{1-3}
		Connectivity threshold $\delta$ & $\{1.6, 1.8, 2.0, 2.2, 2.4\}$& $2.0$\\
		Collision threshold $\omega$ & $\{0.1, 0.2, 0.3, 0.4, 0.5\}$ &$0.4$\\
		Congestion threshold $\lambda$ & $\{0.6, 0.7, 0.8, 0.9, 1.0\}$& $0.9$\\
		Mass factor $d$ & $\{0.3, 0.4, 0.5\}$&$0.5$\\
		Bound decision threshold $q_{b}$ & $\{0.65, 0.70, 0.75, 0.80\}$& $0.75$\\
		Diff. decision threshold $q_{d}$ & $\{0.1, 0.15, 0.2\}$ & $0.1$\\		
		\bottomrule		
	\end{tabular}
	\label{tbl:spec}
\end{table}
\begin{table}
	\centering 
	\caption{Parameter setting associated with meta-navigation-function-based coordinators of robots.}
	\begin{tabular}{lll}
		\toprule  
		{\bfseries Parameters}&{\bfseries Variation sets} &{\bfseries Optimal val.}\\\cmidrule{1-3}
		Kernel factor $\beta$ & $\{7.0, 7.5, 8.0, 8.5\}$& $8$\\
		Attractive factor $\lambda_{1}$& $\{0.4, 0.5, 0.6, 0.7\}$ & $0.5$\\
		Repulsive factor $\lambda_{2}$& $\{0.08, 0.09, 0.01, 0.11\}$ & $0.10$\\
		Associative factor $\lambda_{3}$& $\{0.001, 0.002, 0.003\}$ & $0.001$\\	
		\bottomrule		
	\end{tabular}
	\label{tbl:mnf}			
\end{table}
\yy{\begin{prop}[Computational Complexity]
		Let $|\mathcal{R}|$ be the cardinality of ad-hoc robot network $\mathcal{R}$. Then, the computational complexity of Algorithm \ref{alg:TPP} is $\mathscr{O}(|\mathcal{R}|^{2})$.
	\end{prop}
	\begin{proof}
		The line 2 of the algorithm encompasses ${|\mathcal{R}| \choose 2}$ operations. For a particular $x< |\mathcal{R}|$, the line 8 includes ${x \choose 2}$ operations. Thus, the bottleneck is the line 2. The expansion of ${|\mathcal{R}| \choose 2}$, for large $|\mathcal{R}|$, using the Sterling's approximation yields
		${|\mathcal{R}| \choose 2} = \frac{|\mathcal{R}|!}{2!(|\mathcal{R}|-2)!}
		\approx \frac{e^4}{2}|\mathcal{R}|^2$,
		which completes the proof.
\end{proof}}
\section{Results\protect\footnote{The simulations are performed on a Windows 10 64x machine supported by a Core i7 1.80 GHz processor, 8GB RAM, and an Intel UHD Graphics 620. The following specific libraries are employed in the course of all performed simulations on Python 3.7.3: Tensorflow and Keras 2.5.0, Scikit-Learn 0.24.2}}
\label{sec:res}
\subsection{Setup}
\begin{table} 
	\centering 
	\caption{Cross-validation-based tuned values of B-BGMMs hyperparameters.} 
	\begin{tabular}{lll}
		\toprule  
		{\bfseries Parameters}& {\bfseries Variation sets} &{\bfseries Optimal val.}\\\cmidrule{1-3}
		Maximum \# of components & $\{10, 15, 20, 25\}$& $15$\\
		Covariance type & $\{\text{Full}, \text{Tied}\}$ &Full\\
		\# of initializations & $\{10, 15, 20, 25\}$ &$10$\\
		Weight concen. prior type & Dirichlet $\{\text{dist.}, \text{proc.}\}$&Dirichlet dist.\\
		Weight concen. prior $\gamma_{0}$ & $\{950, 1000, 1050, 1100\}$&$1000$\\
		Maximum iterations & $\{4000, 5000, 6000\}$& $5000$\\		
		\bottomrule		
	\end{tabular}
	\label{tbl:opt}
\end{table}
\begin{figure}
	\centering\includegraphics[scale=0.8]{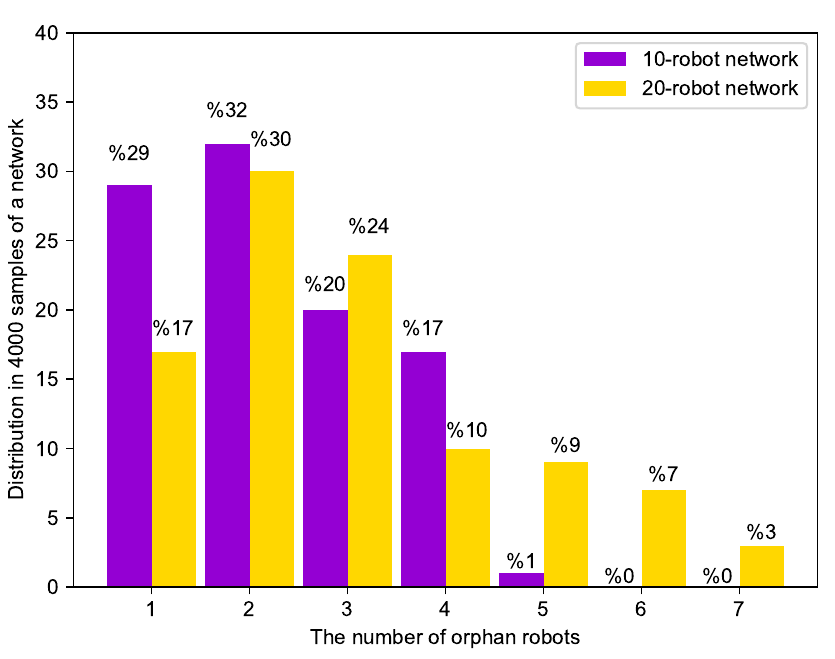}
	\caption{Orphan robot cardinality distributions due to faults applied to random robots\xx{.}}
	\label{fig:dist_net}
\end{figure}
\begin{figure}
	\centering
	\begin{subfigure}[b]{0.3\textwidth}
		\centering
		\includegraphics[width=\textwidth]{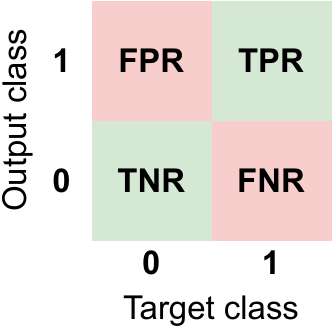}
		\caption{{Template\xx{.}}}    
		\label{fig:cm-temp}
	\end{subfigure}
	\begin{subfigure}[b]{0.3\textwidth}  
		\centering 
		\includegraphics[width=\textwidth]{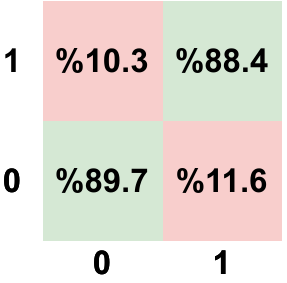}
		\caption{{The 10-robot network\xx{.}}}    
		\label{fig:cm_10}
	\end{subfigure}
	\begin{subfigure}[b]{0.3\textwidth}   
		\centering 
		\includegraphics[width=\textwidth]{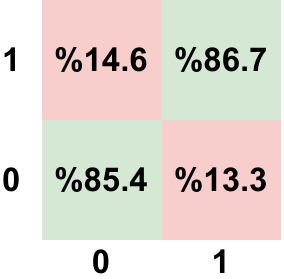}
		\caption{{The 20-robot network\xx{.}}}    
		\label{fig:cm_20}
	\end{subfigure}
	\caption{Confusion matrices\xx{.}} 
	\label{fig:cm}
\end{figure}
\begin{table}
	\centering
	\caption{The performance report of our strategy\xx{.}} 
	\begin{tabular}{cccccc}
		\toprule  
		Robot network population& TPR (\%) & TNR (\%)& Balanced accuracy (\%)& Precision (\%)& F1 (\%)\\\cmidrule{1-6}
		10 &88.4&89.7&89.5&92.4&87.4\\
		20 &86.7&85.4&86.5&89.9&85.1\\		
		\bottomrule		
	\end{tabular}
	\label{tbl:res}
\end{table}
We take two datasets\footnote{The generator of the dataset used in this study may be found in	https://git.io/JXKSb.} into account for a 10-robot and a 20-robot network\xx{,} each of which includes 4600 records corresponding to different topological configurations of those networks. Data collection is a low-level operation compared to the intended topology (ir)recoverability prediction process. So, in our simulations, we simply use a shared memory including data of all robots accessible by all of them. As stated before, the input for the topology (ir)recoverability prediction problem \xx{consists of} the snapshots of the locations of robots. Even though the evolutions of those locations, known as coordination, are irrelevant in view of how the Algorithm \ref{alg:TPP} works, one has to take those evolutions into account during simulations to consider physical aspects of robot motions. In particular, we employ meta navigation functions \cite{macktoobian2022meta} of the form
\begin{equation*}
	\label{eq:psi}
		\psi(\bm{q_{i}};\alpha) :=  
		\lambda_{1}\norm[\big]{\bm{q_{i}}-\bm{q^{t}}_{i}}^{2} +
		\frac{\lambda_{2}}{\alpha}\sum_{\mathclap{\rule{0mm}{4mm} j\in[\mathcal{R}\setminus\{i\}]\dot{\cup}\mathcal{O}}}\dfrac{\norm[\big]{\bm{q_{i}}-\bm{q^{t}_{i}}}^{\frac{1}{\alpha}}}{\norm{\bm{q_{i}}-\bm{q_{j}}}^{2}\hfill}+
		\lambda_{3}\norm[\big]{\bm{q_{i}}-\bm{q^{t}_{i}}}^{2}\displaystyle\sum_{\mathclap{k\in[\mathcal{R}\setminus\{i\}]}}\norm{\bm{q_{k}}-\bm{q^{t}_{k}}}^{2},
\end{equation*}
to coordinate robots from one configuration to another. Here, $\bm{q}_{i}$ and $\bm{q}^{t}_{i}$ are the location vector of robot $i$ and its target location, respectively; $\lambda_{1}$, $\lambda_{2}$, and $\lambda_{3}$ are attractive, repulsive, and associative factors of the function; $\mathcal{R}$ is the set of robots, and $\mathcal{O}$ is the set of obstacles that are other robots with respect to robot $i$; $\alpha$ is an adaptive confinement factor. The attraction kernel of the cited meta navigation function is $\omega(\bm{q}) := \beta\norm{\bm{q} - \bm{q^{t}}}$,	where, $\beta$ is a kernel factor. The derivative of $\psi(\bm{q_{i}};\alpha)$ is defined as a velocity profile for robot $i$ to be coordinated according to its target plan. In the setup of the distributed coordination controller associated with each robot, physical aspects of its coordination problem such as safety radius and the setting of artificial potential fields have to be considered\footnote{\xx{There are general guidelines to set the hyperparameters of meta navigation functions (which require no direct optimization) whose details may be found in \cite{macktoobian2022meta}.}}, as specified in Table \ref{tbl:mnf}. Each coordination is continued until either it is completed or a fault occurs. 

As stated before, the source of a fault may be a collision in a very dense area or the congestion of a robot's buffer because of unexpected flows of inward data to that buffer. Once a fault occurs, the states of robots, say, their locations, are taken into account as the aforesaid snapshots to be used in our B-BGMM-based model.
\begin{figure*}
	\centering
	\begin{subfigure}[b]{0.45\textwidth}
		\centering
		\includegraphics[width=\textwidth]{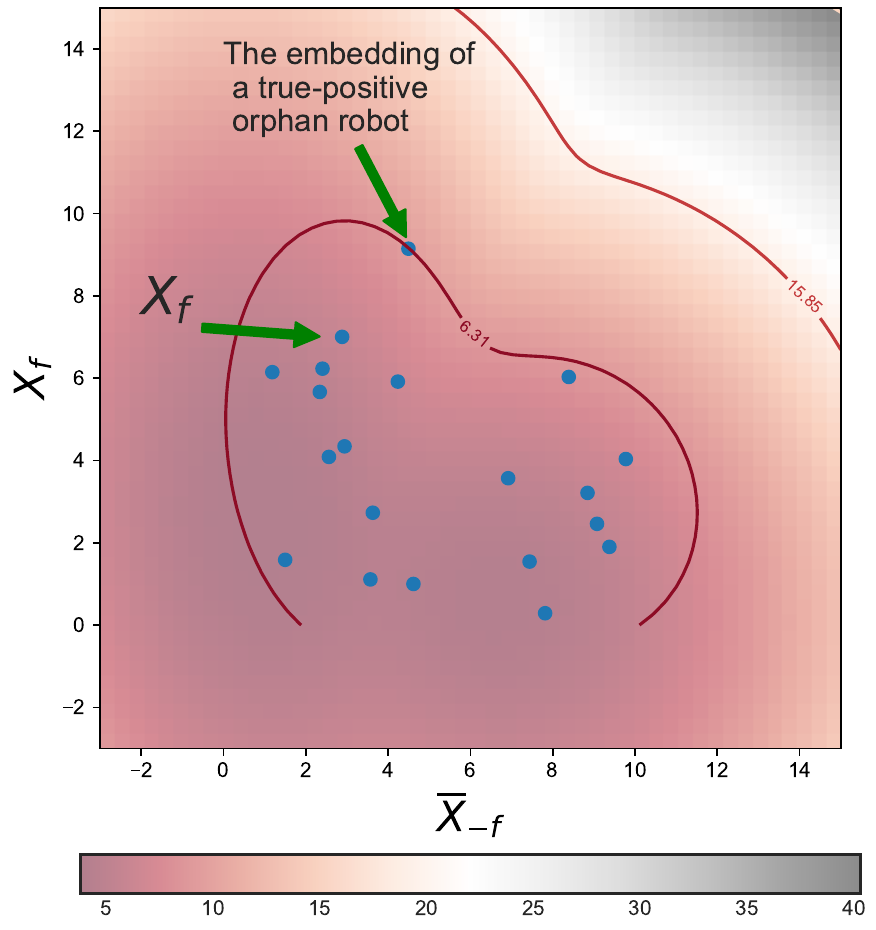}
		\caption[]%
		{{\small A true positive embedding}}    
		\label{dif:tp}
	\end{subfigure}
	\hfill
	\begin{subfigure}[b]{0.45\textwidth}  
		\centering 
		\includegraphics[width=\textwidth]{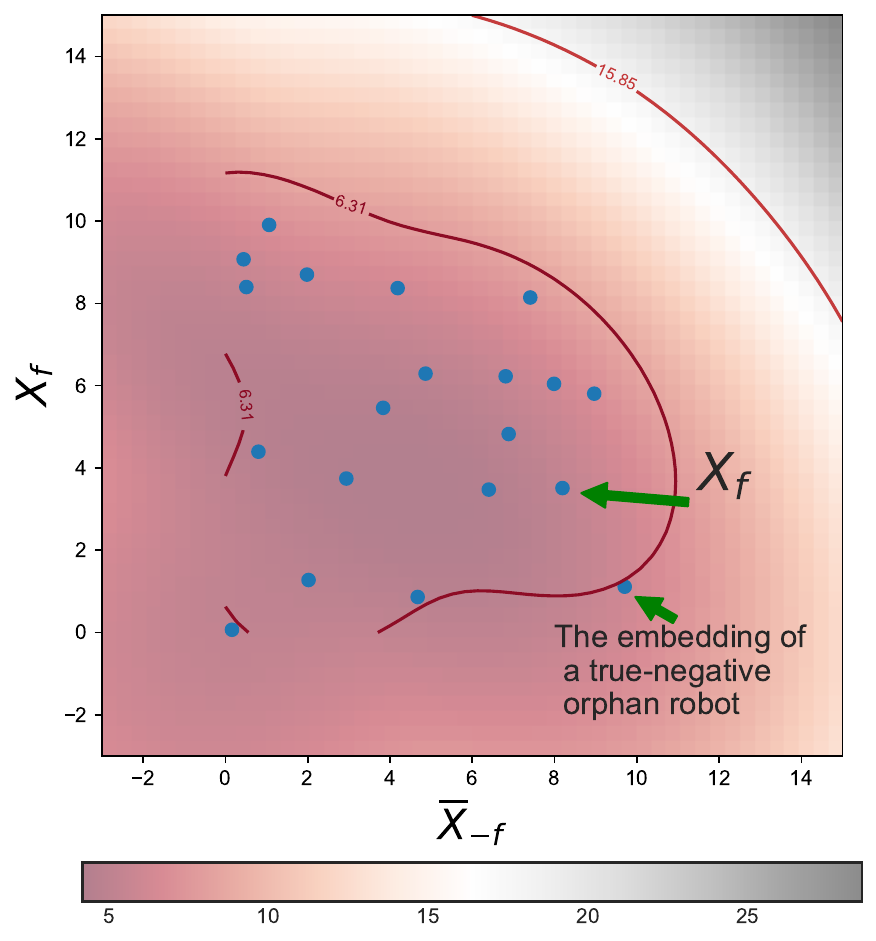}
		\caption[]%
		{{\small A true negative embedding\xx{.}}}    
		\label{fig:tn}
	\end{subfigure}
	\vskip\baselineskip
	\begin{subfigure}[b]{0.475\textwidth}   
		\centering 
		\includegraphics[width=\textwidth]{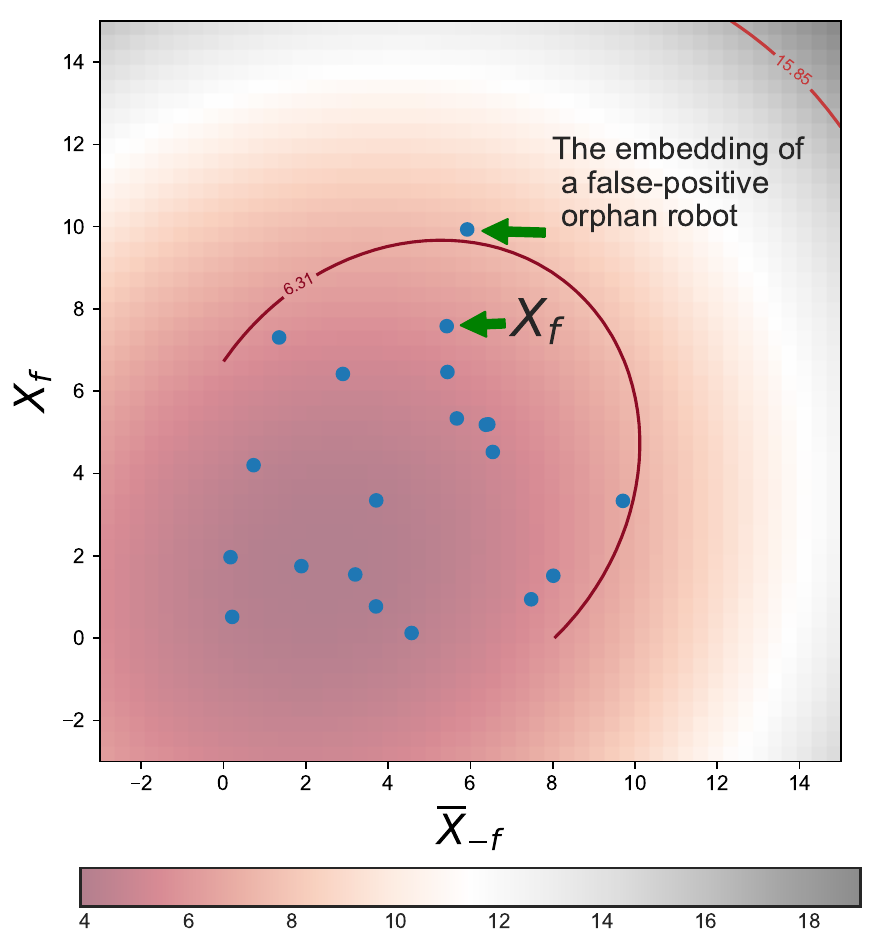}
		\caption[]%
		{{\small A false positive embedding\xx{.}}}    
		\label{fig:fp}
	\end{subfigure}
	\hfill
	\begin{subfigure}[b]{0.475\textwidth}   
		\centering 
		\includegraphics[width=\textwidth]{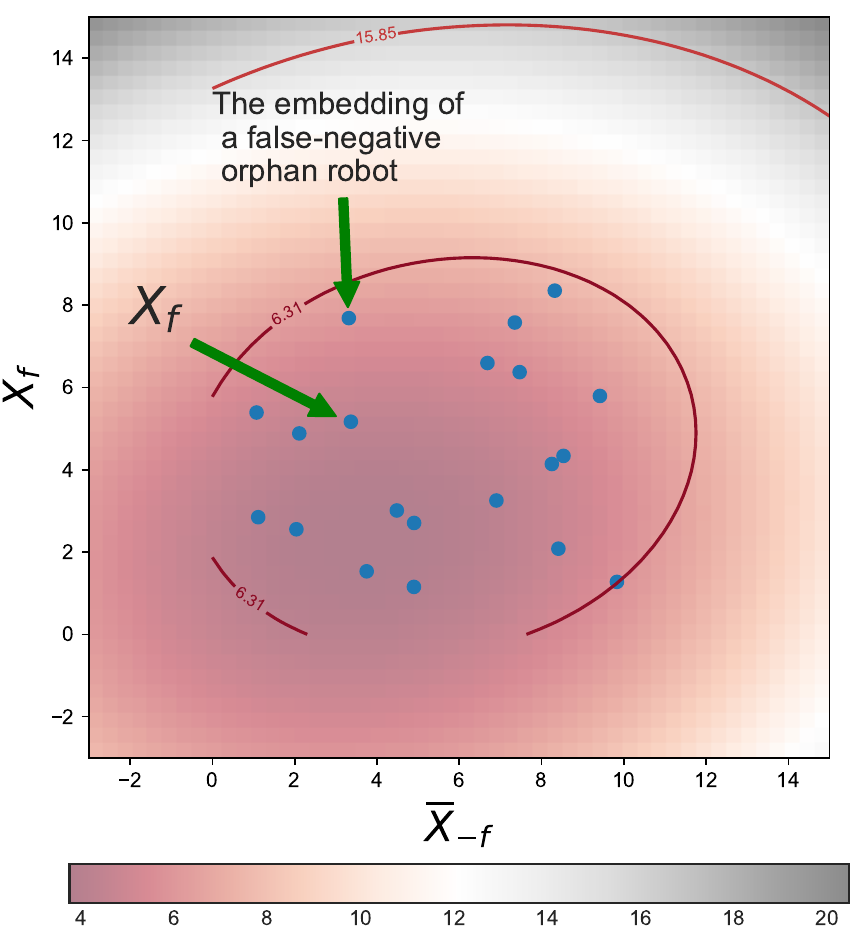}
		\caption[]%
		{{\small A false negative embedding\xx{.}}}    
		\label{fig:fn}
	\end{subfigure}
	\caption[Sample contours of embeddings]
	{\small Sample contours associated with \xx{four} differently-classified embeddings corresponding to the 20-robot network\xx{.} (in each figure, $X_{f}$ is a faulty robot, while $\overline{X}_{-f}$ is its orphan robot according to which the embedding PDF is computed. The numbers on each contour denotes the relative score of that contour regarding its estimated success to participate in the desired prediction task.)} 
	\label{fig:contours}
\end{figure*}

Topologies are computed using OpTopNET \cite{macktoobian2023learning}. \zzz{For the 10-robot and 20-robot networks we test here, OpTopNET instantly computes the optimal topology even on an average machine.} We split the datasets into disjoint pieces to perform training, cross-validation, and testing. The number of records associated with those processes are 4000, 400, and 200, respectively. The values of the parameters associated with our approach are specified in Table \ref{tbl:spec}. We set the values corresponding to connectivity threshold $\delta$, congestion threshold $\lambda$, and collision threshold $\omega$ to 2, 0.9, and 0.4, respectively. In particular, the value of $\delta$ directly depends on the radiation power of robots\footnote{\zzz{The radiation power of a robot corresponds to the strength of electromagnetic waves, such as WiFi, emitted by that robot's antenna for communicational purposes.}}, say, the stronger their antennas are, the larger this threshold will be. Large values of $\delta$ manifest more robot distribution flexibility with respect to a network. Communication faults occur more frequently in sparse networks. 

As already stated in Remark \ref{rem:dense}, the fault bottleneck in our experiments is congestion. Thus, the congestion threshold $\lambda$ is larger than the collision threshold $\lambda$. Mass factor $d$ impacts the radius of all $d$-neighborhoods of a network. Accordingly, the larger the value of $d$ is, the smaller $d$-neighborhoods are. That is because smaller values of $d$ probabilistically increase the radii of $d$-neighborhoods. We set $d$ to 0.5 to achieve balanced neighborhoods for better assessments of congestions. Bound decision threshold $q_b$, set to 0.75, determines the minimum confidence in the estimation of both pre-fault and post-fault predictions. Differential decision threshold $q_d$ is considered to be 0.1 to track the consensus of the aforesaid predictions to distinguish the scenarios in which there are discrepancies between their probabilistic votes.  

\yy{We use grid-search cross-validation to tune all hyperparameters associated with our method. In particular, each optimal value is found with respect to a specific variation set. Tables \ref{tbl:spec}, \ref{tbl:mnf}, and \ref{tbl:opt} report the conducted hyperparameter cross-validation corresponding to topological characterizations, meta navigation function, and BGMMs of the proposed scheme, respectively.} In particular, we \xx{do not} impose any restriction on the full covariance matrices of our B-BGMMs not to restrict the contour shapes of the models' PDFs. In scenarios where data \xx{are} inherently generated from Gaussian distributions, one can efficiently find the optimal number of components for each B-BGMM using the Bayesian inference criterion \cite{schwarz1978estimating}. However, we generally assume that robots are uniformly distributed in a network. \xx{That explains why} the method above may not be applicable to our approach. Alternatively, we select an upper-bound for the maximum number of components, so that expectation-maximization algorithm converges to any number of components up to that upper-bound. The cross-validation process asserts that the optimal upper-bound in our setting is 15. Parameter initializations associated with mean and covariance matrices may be either random or based on K-Means algorithm. K-Means often imposes some classification biases on input data that are not desirable in the case of uniform distributions. So, cross-validation votes for the optimality of random initializations. The more diverse the components of a B-BGMM, corresponding to pre-fault or post-fault predictions, are, the higher the probability of correct detection of topology (ir)recoverability will be. \xx{Thus}, the cross-validation process, acknowledging this requirement, selects Dirichlet distribution over Dirichlet process as the weight concentration prior type of our models. 
\subsection{Performance}
Faults are randomly applied to robots resulting in the orphan robot distributions exhibited in Figure \ref{fig:dist_net}. One observes that the majority of faults lead\xx{s} to the creation of no more than \xx{four} orphan robots. This observation can be justified by the fact that robot coordinates are generated by a uniform distribution. So, the robots may not probabilistically shape dense local populations around any peer.   
\begin{figure*}
	\centering
	\begin{subfigure}[b]{0.49\textwidth}
		\includegraphics[width=\textwidth]{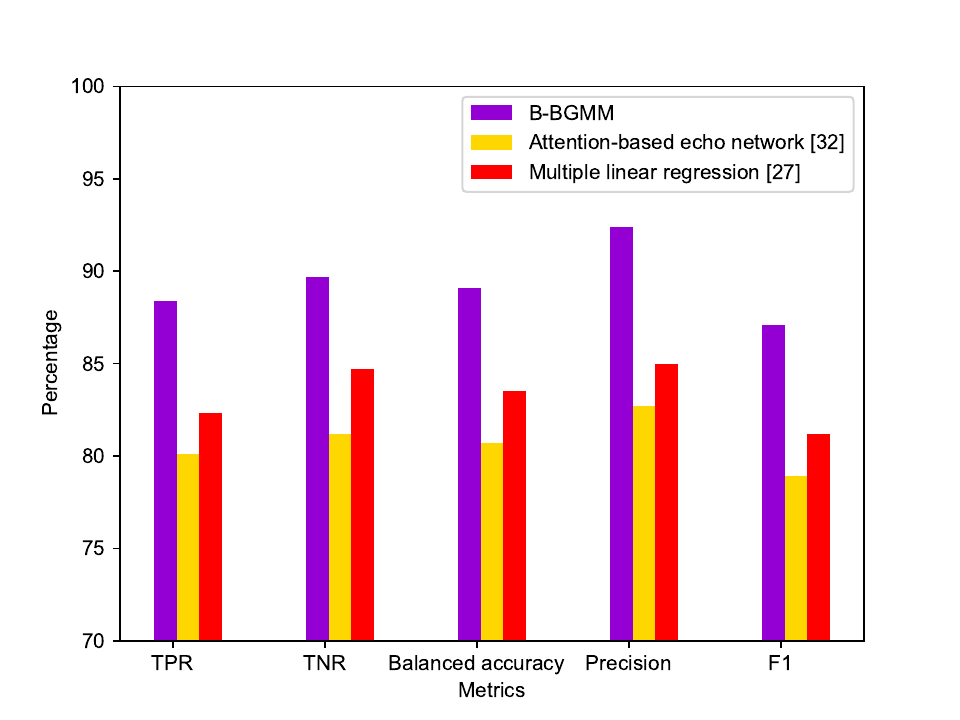}
		\caption[]%
		{{\small The 10-robot network case\xx{.}}}    
		\label{fig:ag10}
	\end{subfigure}
	\begin{subfigure}[b]{0.49\textwidth}  
		\includegraphics[width=\textwidth]{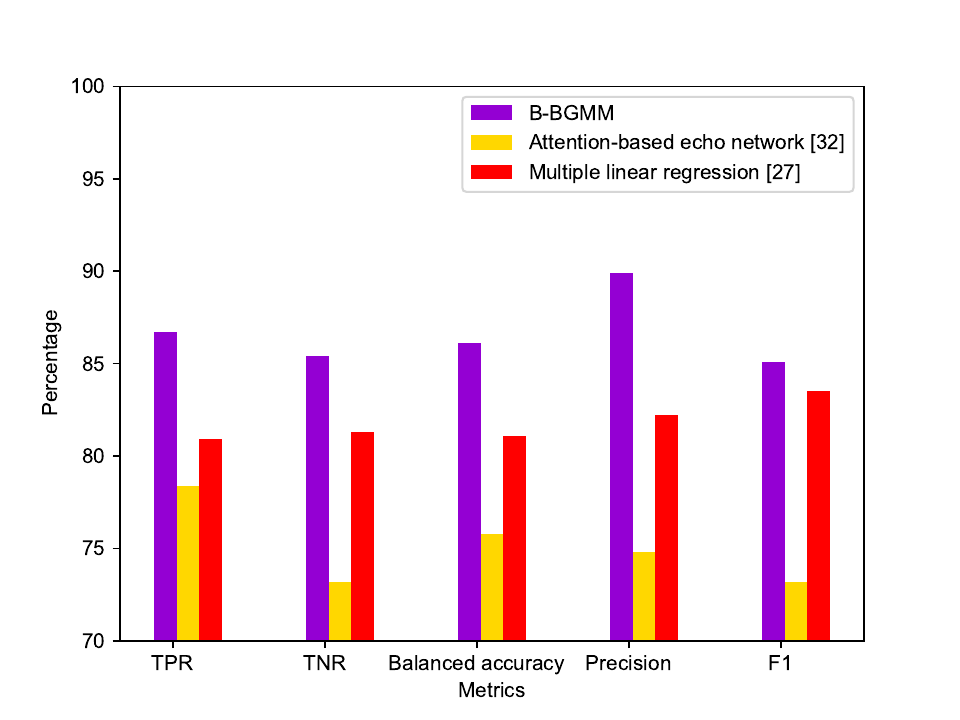}
		\caption[]%
		{{\small The 20-robot network case\xx{.}}}    
		\label{fig:ag20}
	\end{subfigure}
	\caption[Sample contours of embeddings]
	{\small Comparative performance results of our proposed B-BGMM-based model, the attention-based echo network, and the multiple-linear-regression-based model\xx{.}} 
	\label{fig:ag}
\end{figure*}
\begin{table}
	\centering\caption{Logloss nominal performance report of the proposed strategy compared to the attention-based echo network model and multiple linear regression model}
	\begin{tabular}{ccccccc}
		\toprule  
		&\multicolumn{3}{c}{10-robot network}&\multicolumn{3}{c}{20-robot network}\\
		\cmidrule(lr){2-4}\cmidrule(lr){5-7}
		& \textbf{B-BGMM} & AEN& MLR & \textbf{B-BGMM} & AEN& MLR\\
		\cmidrule(lr){1-1}\cmidrule(lr){2-2}	\cmidrule(lr){3-3}\cmidrule(lr){4-4}\cmidrule(lr){5-5}\cmidrule(lr){6-6}\cmidrule(lr){7-7}	
		Positive-class average logloss&\textbf{0.103}&0.342&0.255&\textbf{0.123}&0.418&0.309\\ 
		Negative-class average logloss&\textbf{0.106}&0.338&0.249&\textbf{0.119}&0.458&0.312\\\cmidrule{1-7}
		Total average logloss&\textbf{0.105}&0.340&0.252&\textbf{0.121}&0.438&0.311\\
		\bottomrule		
	\end{tabular}
	\label{tbl:logloss}
\end{table}
\begin{table}
	\centering\caption{Logloss noisy performance report of the proposed strategy compared to the attention-based echo network model and multiple linear regression model}
	\begin{tabular}{ccccccc}
		\toprule  
		&\multicolumn{3}{c}{10-robot network}&\multicolumn{3}{c}{20-robot network}\\
		\cmidrule(lr){2-4}\cmidrule(lr){5-7}
		& \textbf{B-BGMM} & AEN& MLR & \textbf{B-BGMM} & AEN& MLR\\
		\cmidrule(lr){1-1}\cmidrule(lr){2-2}	\cmidrule(lr){3-3}\cmidrule(lr){4-4}\cmidrule(lr){5-5}\cmidrule(lr){6-6}\cmidrule(lr){7-7}	
		Positive-class average logloss&\textbf{0.108}&0.351&0.262&\textbf{0.130}&0.429&0.314\\ 
		Negative-class average logloss&\textbf{0.110}&0.346&0.252&\textbf{0.126}&0.466&0.320\\\cmidrule{1-7}
		Total average logloss&\textbf{0.109}&0.349&0.257&\textbf{0.128}&0.447&0.317\\
		\bottomrule		
	\end{tabular}
	\label{tbl:loglossNoise}
\end{table}
The performance of the binary classification problem of ours, with respect to both 10-robot and 20-robot networks, \xx{is} reported in confusion matrices \xx{in} Figure \ref{fig:cm}. In particular, a template confusion matrix, as a guideline to read further matrices of this research, is depicted in Figure \ref{fig:cm-temp}. The entries of this template matrix are defined based on the notions below.
\begin{itemize}
	\item A true positive (TP) is a network topology that is predicted
	to be recoverable, and it is actually
	recoverable.
	\item A false positive (FP) is a network topology that is predicted
	to be recoverable, but it is actually irrecoverable.
	\item A true negative (TN) is a network topology that is predicted
	to be irrecoverable, and it cannot be recovered.
	\item A false negative (FN) is a network topology that is predicted
	to be irrecoverable, despite its recoverability.
\end{itemize}
In this regard, standard ratios of the above notions, i.e., TPR, TNR, FPR, FNR, constitute the entries of a confusion matrix. Figure\xx{s} \ref{fig:cm_10} and \ref{fig:cm_20} render the classification performance of our approach with respect to the 10-robot and 20-robot networks, respectively. One observes that both TPR and TNR factors are efficiently high, exhibiting the effective classifications realized by our strategy. Table \ref{tbl:res} also states the values corresponding to balanced accuracy\footnote{Balanced accuracy, as the	average of TPR and TNR, measures the predictive efficiency of a binary classification solution.}, precision, and F1 score\footnote{\xx{F1 score is the harmonic mean of precision and recall.}} of the classification results for both networks.

The embedding space of a B-BGMM is a mean of probabilistic inference about the topology (ir)recoverability of a fault. Contours of such embeddings convey interesting information about their underlying PDFs. For example, Figure \ref{fig:contours} includes the embedding contours corresponding to four different classification scenarios associated with the 20-robot network. The visual inspection of the distances in its sub-figures provides sufficient intuition about how the topology (ir)recoverability prediction problem may be solved by embedding data into B-BGMMs and inferring their probabilistic dynamics in view of a network topology. One may note that the contours of Figure \ref{fig:contours} are two-dimensional because of the fact that we used B-BGMMs as the computational kernel of our solution. If one uses multivariate BGMMs, then the resulting hyperplanes may include more potential solutions for post-fault predictions. However, as already discussed, it will be yielded at the cost of increasing the learning time of the algorithm that may not be favorable particularly in the case of large-scale ad-hoc robot networks. 

\zzz{To illustrate the advantages and merit of our B-BGMM-based model, we compare the results above to the results of two data-driven methods that are among the best for topology prediction, i.e., attention-based echo network \cite{liu2019attention} and multiple linear regression \cite{das2017fault}, applied to the same setting of the 10-\xx{robot} and 20-robot networks above. First we report the setup of the models synthesized based on the two cited methods. Attention-based echo network essentially employs a ridge regression engine with recurrent feedback to its hidden layers. Its updating process follows the rules $h(t+1) = (1-\kappa)\phi(W_{\text{in}}u(t+1)+W_{\text{in}}h(t)+W_{\text{back}}\hat{y}(t)+\tau)+\kappa h(t)$ and $y(t+1) = W_{\text{out}}(h(t+1);u(t+1))$.	Here, $\tau$ is an \xx{independent and identically distributed} noise, $u(t)$ is the observed topology as the input for the network at time $t$; $y$ and $\hat{y}$ are the post-fault predicted and target (ir)recoverabilities of the input topology; $h(t)$ is the current state of the network before a fault occurrence. $W_{\text{in}}$ is a randomly-initialized weight matrix; $W_{\text{out}}$ is the weight matrix determined by the ridge regression training phase; $W_{\text{back}}$ denotes the matrix of feedback weights; $\phi(\cdot)$ is a sigmoid activation function; and $\kappa$ is a leaky rate that is set to $0.1$. Taking mean square error as the training metric, we split the dataset into three partitions, say, 80\% for training, 10\% for validation, and 10\% for testing. We also perform $k$-fold Monte Carlo cross-validation where $k = 10$. We train the kernel of the model using ridge regression function of Scikit-Learn library \cite{pedregosa2011scikit}. We choose fit intercept option for estimation purposes. Since the base coordinate has to be similar to those of the previous tests, we confine the coordinates to be positive. We also set normalization flag to true so that distances from the base coordinate do not enforce any bias on the training process. We also set $\mathcal{L}_{2}$ norm to $1.5$ to achieve better regularization.}

\zzz{The multiple-linear-regression-based model \cite{das2017fault}, which was primarily proposed to be used for preserving connectivity for underwater networks, inherently works with 3D data. To adopt it to be applicable to our 2D tests, a pre-processing step is taken into account to convert our data to the desired geographical format of that network. Then, we seek the regression dependence of post-fault connectivity flag vector $A$, each component of which associates with one of our simulated robots, on geographical location vector $B$. Thus, for $n$ observations, the regression model reads as $A = \theta_{0} + \theta_{1}B_{1} + \cdots + \theta_{n}B_{n}$, where vector $\theta$ represents regression co-efficients. In the training phase, we use instances of the linear regression utility of Scikit-Learn library. The setup of fit intercept, base coordinate, and normalization flag resembles that of the attention-based echo network above. Mean square error is also the metric of interest in this case.}

\zzz{As a result, Figure \ref{fig:ag} illustrates the superior performance of the proposed B-BGMM-based model of ours corresponding to both the 10-robot (see, Figure \ref{fig:ag10}) and the 20-robot (see, Figure \ref{fig:ag20}) networks with respect to all of the investigated metrics.} \yy{
	An advantage of the proposed B-BGMM-based method over the attention-based echo network \cite{liu2019attention} is its unsupervised nature. In particular, our B-BGMM model benefits from lazy evaluations which make it a faster decision maker compared to that of the attention-based echo network model. Another feature of the B-BGMM-driven scheme is that it transforms local data to nonlinear topological representations in Gaussian spaces. Since topology of a network is essentially a geometrical notion, this feature yields better results compared to the multiple linear regression \cite{das2017fault} in discovering more complex patterns of topology (ir)recoverability in faulty ad-hoc robot networks.}

\yy{The performance superiority of the proposed scheme, compared to the baseline models, may also be observed according to the logloss report depicted in Table \ref{tbl:logloss} for both 10-robot and 20-robot networks. Another advantage of our model is its robustness in presence of noisy data. For this purpose, we perturbed the position $q$ of each robot by addition of some white noise according to the rule 
	\begin{equation*}
		q \rightarrow q+\text{sample}(\mathcal{U}[q-0.1q, q+0.1q]), 
	\end{equation*} 
	where $\mathcal{U}(\cdot, \cdot)$ denotes a uniform distribution bounded by its arguments. As Table \ref{tbl:loglossNoise} illustrates, logloss values corresponding to the B-BGMM model remain close to the logloss values associated with the nominal case, represented in Table \ref{tbl:logloss}.}
\section{Concluding Remarks}
\label{sec:conc}
When some faults occur in a subset of robots in an ad-hoc robot network, its nominal topology may be perturbed in view of feasibility of communications. \xx{Hence}, one seeks the prediction of topology (ir)recoverability for the post-fault network. This paper solves this problem based on a data-driven approach. A double-pathway predictor, governed by a Bayesian inference engine, is developed that processes the probability density functions associated with bivariate Bayesian Gaussian distributions (B-BGMMs) of the network's topological embeddings. The efficiency of the obtained numerical results manifests the success of our model to hit that mark. 

\xx{This research, despite its success in topology (ir)recoverability prediction, may lead to the emergence of further questions to even improve its formalism and results to a greater extent as follows.  (i) We used B-BGMMs for computational reasons as well as binary relations between orphan robots and their neighbors. However, alternative partial or total embodiments of multivariate BGMMs may reveal more unknown predictive capabilities for this model. Furthermore, one may intend to establish some meta learning approaches for the decision-making step of the strategy to (at least partially) relax its current set of hyperparameters. (2) Our strategy cannot be applied to split-and-merge maneuvers of ad-hoc robot networks when such splits are too far from each other so that they cannot be linked to each other by at least two of their robots. That is because the dataset associated with a total network has to be partitioned to various disjoint subsets from one scenario to another. This remarkably reduces the number of samples and their correlations in a scenario. As a result, any learning process becomes extremely inefficient, if not totally infeasible. The interpretation of a split-and-merge scenario in view of our approach is the irrecoverability of its corresponding pre-fault topology. A generalization of our strategy to cover such split-and-merge maneuvers is another potential continuation of this work.} \yy{(3) The distribution of robot locations in our simulations is uniform. For sufficiently-small values associated with connectivity threshold $\delta$, one or more robots may act as a hub connected to many other robots, similarly to power-law-driven networks. So, the reported results are not sensitive in that regard. However, in the investigated 20-robot network, the largest network in our tests, the severity of the quoted power-law-driven centrality is not noticeable. Moving toward expansion of our method to larger networks, while maintaining high accuracies, is another venue for future research.} \qq{(4) Finally, further experiments may be conducted to validate the robustness of the proposal across diverse topological configurations and failure/attack strategies.}
\section*{Acknowledgement}
The authors appreciate the constructive comments of anonymous reviewers which led to this improved exposition.
\bibliographystyle{IEEEtran} 
\bibliography{references}
\end{document}